\documentclass[sigconf, nonacm]{acmart}

\usepackage{etoolbox}

\usepackage{setspace}

\newtheorem{definition}{Definition}

\newcommand{\stitle}[1]{\vspace*{0.4em}\noindent{\bf #1.\/}}
\newcommand{\sstitle}[1]{\vspace*{0.4em}\noindent{\bf #1:\/}}

\newcommand{\trim}{\vspace{-2mm}}
\newcommand{\ttrim}{\vspace{-1mm}}

\newcommand{\deng}[1]{{#1}}

\usepackage[linesnumbered]{algorithm2e}
\usepackage{caption}
\usepackage{subcaption}
\usepackage{bm}
\usepackage{pifont}
\usepackage{soul}

\usepackage{colortbl}
\usepackage{makecell}
\usepackage{multirow}

\usepackage[english]{babel}
\usepackage{hyperref}
\usepackage{tabularx}

\usepackage{marvosym}

\RestyleAlgo{ruled}

\newcommand{\squishlist}{
 \begin{list}{$\bullet$}
 { \setlength{\itemsep}{0pt}
   \setlength{\parsep}{3pt}
   \setlength{\topsep}{3pt}
   \setlength{\partopsep}{0pt}
   \setlength{\leftmargin}{1.2em}
   \setlength{\labelwidth}{1em}
   \setlength{\labelsep}{0.6em}
 }
}
\newcommand{\squishend}{
 \end{list}
}

\addto\extrasenglish{

}

\def\equationautorefname~#1\null{Equation~(#1)\null}

\newtheorem{MyTheo}{Theorem}

\AtBeginDocument{%
  \providecommand\BibTeX{{%
    \normalfont B\kern-0.5em{\scshape i\kern-0.25em b}\kern-0.8em\TeX}}}

\ccsdesc[500]{Information systems~Data management systems}
\ccsdesc[500]{Computing methodologies~Artificial intelligence}

\keywords{Vector database, Large language model, Machine learning systems}

\begin{document}

\title{AlayaDB: The Data Foundation for Efficient and Effective \\ Long-context LLM Inference}

\author{
Yangshen Deng*, Zhengxin You*, Long Xiang*, Qilong Li, Peiqi Yuan, Zhaoyang Hong, \\ 
Yitao Zheng, Wanting Li, Runzhong Li, Haotian Liu,
Kyriakos Mouratidis, Man Lung Yiu, \\ 
Huan Li, Qiaomu Shen, Rui Mao, Bo Tang\textsuperscript{\Letter} \\
Corresponding  email:  research@alayadb.ai
}

\renewcommand{\shortauthors}{Y. Deng, Z. You, L. Xiang, et al.}
\renewcommand{\shorttitle}{AlayaDB}

\begin{abstract}
AlayaDB is a cutting-edge vector database system natively architected for efficient and effective long-context inference for Large Language Models (LLMs) at \textsf{AlayaDB AI}.
Specifically, it decouples the KV cache and attention computation from the LLM inference systems, and encapsulates them into a novel vector database system.
For the Model as a Service providers (MaaS), AlayaDB consumes fewer hardware resources and offers higher generation quality for various workloads with different kinds of Service Level Objectives (SLOs), when compared with the existing alternative solutions (e.g., KV cache disaggregation, retrieval-based sparse attention).
The crux of AlayaDB is that it abstracts the attention computation and cache management for LLM inference into a query processing procedure,
and optimizes the performance via a native query optimizer.
In this work, we demonstrate the effectiveness of AlayaDB via (i) two use cases from our industry partners, and (ii) extensive experimental results on LLM inference benchmarks.

\end{abstract}

\maketitle

\let\thefootnote\relax\footnotetext{ *~ These authors contributed equally to this work. \\
\textsuperscript{\Letter} Corresponding author: Prof. Bo Tang.}

\section{Introduction} \label{sec:intro}
Large Language Models (LLMs) have been widely used in various real-world applications such as personal assistants~\cite{chatgpt,kimi,gemini,deepseek,qwen}, search engines~\cite{bing,perplexity,google,amazon}, code generators~\cite{deepseekcoder,copilot,cursor,qwen} and document analyzers~\cite{financial,document}.
Efficient and effective LLM inference is an open problem in the industry~\cite{llama-262k,gemini15,internlm1m,yi-6b-200k}, especially for long-context (e.g., millions of tokens) inference.
In particular, the performance of LLM inference systems is evaluated by three metrics: (1) \textsf{inference latency}, the end-to-end time cost for user tasks, (2) \textsf{generation quality}, the capabilities of LLM in various workloads, and (3) \textsf{GPU memory consumption}, the used hardware resources for the user tasks.

Many solutions have been proposed to optimize these metrics in long-context LLM inference.
They can be classified into three categories: (i) coupled architecture; (ii) KV cache disaggregation; and (iii) retrieval-based sparse attention.
vLLM~\cite{vllm}, SGLang~\cite{sglang} and HuggingFace transformers~\cite{transformers} are the most widely-used LLM inference systems in (i) coupled architecture.
LLM model computation and KV cache management are tightly coupled in these systems.
These systems achieve high generation quality as they use a full attention mechanism.
Mooncake~\cite{mooncake} and LMCache~\cite{lmcache,cachegen} are representative LLM inference systems in (ii) KV cache disaggregation.
They store the KV cache of contexts in external storage and reuse them among different LLM inference instances.
Thus, the inference latency of these systems is improved as it reuses the KV cache and reduces the expensive computations (e.g., inner product and softmax).
Recently, retrieval-based sparse attention solutions have been proposed (e.g., InfLLM~\cite{infllm} and RetrievalAttention~\cite{RA}) to alleviate the large GPU memory consumption of these systems in both (i) and (ii).
The core idea behind them is the sparse attention mechanism, i.e., only a subset of critical key and value tokens are selected to perform the attention computation.
Unfortunately, existing systems cannot simultaneously optimize the three aforementioned performance metrics, as we will elaborate in Section~\ref{sec:motivation}.

At \textsf{AlayaDB.AI}, we designed an LLM-native vector database AlayaDB to overcome the limitations of existing LLM inference systems/solutions and enable efficient and effective long-context inference in LLM era.
Specifically, for Model as a Service (MaaS)~\cite{Gan_2023} providers, the SLOs of different kinds of workloads indicate their requirements for the inference latency.
Thus, the core challenge of AlayaDB is solving a bi-objective optimization problem, i.e., \emph{meet the SLOs of different workloads by consuming less GPU memory and offering higher generation quality simultaneously.}
The core idea of AlayaDB is to decouple both KV cache and attention computation and to encapsulate them into a monolithic vector database.
The major benefits of the novel disaggregation level are three-fold.

\squishlist

\item \stitle{Lightweight LLM Inference System} The cache management and attention computation can be separated from the LLM inference engine, which lightens its burden.

\item \stitle{Interface Simplification} It simplifies the interface between LLM inference engine and KV cache service by only returning the attention result, instead of the KV cache content.

\item \stitle{Co-optimization Opportunity} It sheds light on co-optimizing attention computation and KV cache management in a monolithic vector database together.  

\squishend

At a high level, AlayaDB's role in LLM inference is comparable to the role of traditional databases~\cite{aurora, snowflak, postgres, GreenPlum, TiDB, system-R} in web applications.
Specifically, the LLM application developers only need to pay attention to the logic of their applications while AlayaDB offers efficient long-context management from their developed LLM applications.
This is analogous to web application developers focusing on the logic of their applications and leaving efficient data management to a traditional relational database.

To achieve the above vision, there are three design goals of AlayaDB: (i) ease-to-use, (ii) high generation quality, and (iii) good efficiency.
AlayaDB employs a novel system architecture and introduces end-to-end optimizations.
Firstly, it provides simple-yet-powerful abstractions and  APIs, which are compatible with the software ecosystem of LLMs.
Secondly, it handles sparse attention computation as a vector search query.
To improve the generation quality and reduce the memory consumption simultaneously, AlayaDB defines a novel query type, i.e., dynamic inner product range query (DIPR), which overcomes the limitations of the traditional top-$k$ query.
To accelerate query processing, AlayaDB includes a native query optimizer, which selects the best execution plan for efficient vector search.
Last but not least, a suite of optimization techniques (from algorithm-side to index-side, from computation to storage) has been employed in AlayaDB.

Compared to existing LLM inference systems, AlayaDB enjoys low latency, high generation quality, and low resource consumption at the same time from long-context inference. 
Our experience shows that AlayaDB greatly lowers the cost of hardware resources for handling long contexts and lightens the labor for optimizing the LLM infrastructure.
AlayaDB has already been used to support several online LLM services including chatting apps and knowledge-base QA services in our industry partners.

To sum up, the technical contributions of AlayaDB are as follows.

\squishlist

\item \textbf{Novel Decoupling Level for LLM Inference Systems:} We classify existing LLM inference solutions into three categories and analyze their limitations to handle the challenges of long-context LLM inference.
Then, we decouple the KV cache and attention computation from the LLM inference system and encapsulate them into a novel vector database system.

\item \textbf{Dynamic Vector Search Query for Sparse Attention:} We analyze the internal characteristics of sparse attention in various LLM benchmarks and real-world applications, then propose a novel dynamic vector search query, i.e., Dynamic Inner Product Range  (DIPR), to capture the dynamic nature of sparse attention, which overcomes the limitation of traditional top-$k$ query.

\item \textbf{AlayaDB System Architecture and Implementation:} We architect and implement AlayaDB for efficient and effective long-context inference. It consists of user interface, query processing engine, and vector storage engine.
AlayaDB has been used in several LLM applications by our industry partners.
To the best of our knowledge, it is the first vector database natively built for LLM inference.

\item \textbf{Extensive Evaluation:}  We conduct in-depth evaluation of AlayaDB. The results show that it is able to reduce resource consumption and offer better generation quality while guaranteeing the SLOs for various LLM workloads. 
\squishend

The remainder of the paper is organized as follows.
Section~\ref{sec:llminfer} introduces the LLM inference procedure;
Sections~\ref{sec:motivation} and~\ref{sec:alayadb} present the motivation, design goals and architecture of AlayaDB;
Sections~\ref{sec:interface}, ~\ref{sec:query}, and~\ref{sec:opt} describe AlayaDB's components;
Section~\ref{sec:applications} elaborates two use cases of AlayaDB;
Section~\ref{sec:exp} presents the experimental study results and Section~\ref{sec:con} concludes this work.

\begin{figure*}
 \small
 \centering
 \includegraphics[width=0.95\textwidth]{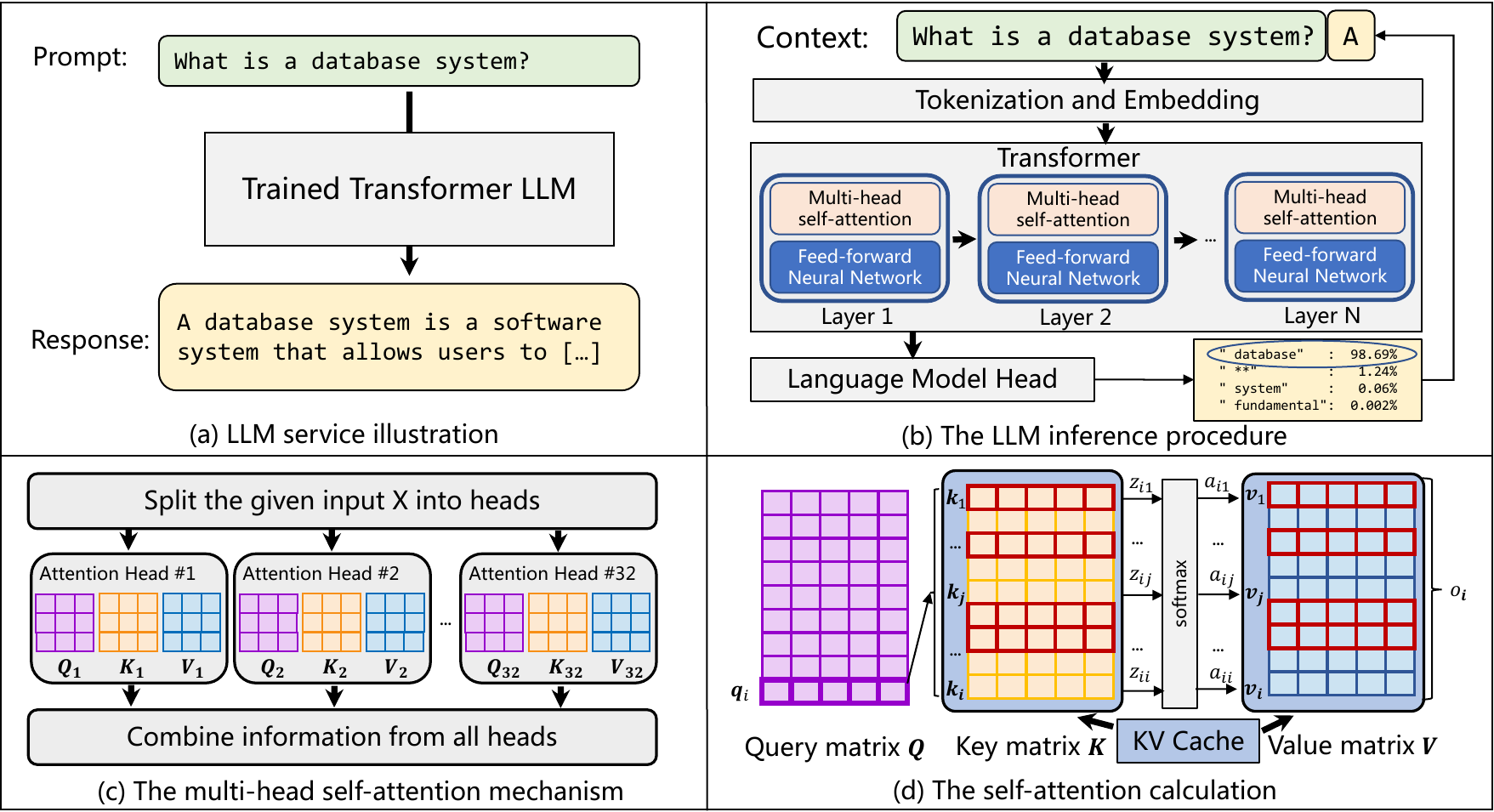}
 \caption{The concepts and illustrations of LLM inference}  \label{fig:llminfer}
\end{figure*}

\section{LLM Inference}\label{sec:llminfer}

A large language model (LLM) is a deep neural language model with billions of parameters.
The decoder-only transformer is the most prevalent architecture in LLMs, such as GPT~\cite{chatgpt}, Llama~\cite{llama}, and Qwen~\cite{qwen}.
Given a well-trained LLM model, the LLM inference generates the text in response to the user input prompt,
as shown in Figure~\ref{fig:llminfer}(a).
Actually, it generates the tokens in response text one by one.
Each token generation is a forward pass of the language model.
The generated token will be appended to the end of the input prompt to form the new context. 
The context is used to generate the next token by following the same forward pass of the language model.
The generation procedure terminates when a special token \texttt{<eot>} (end of text) is generated or the generated text reaches the predefined maximum length.

Figure~\ref{fig:llminfer}(b) depicts the major components in the LLM inference procedure. 
For the input context of LLM, it first breaks down the text into small chunks (i.e., tokens), then turns tokens into numeric representations to capture the meaning via the tokenization and embedding modules.
The embeddings of the context are the input of the transformer model, which consists of a stack of transformer layers that do all the processing.
The output of the transformer is probability scores for what the most likely next token is via the language modeling head.
Transformer LLMs include a stack of transformer layers, e.g., Llama 3.1 has 32 layers.
Each transformer layer processes its inputs and passes the results to the next layer.
For each transformer layer, it has two successive modules: (i) self-attention module, and (ii) feed-forward neural network.
The feed-forward neural network emphasizes the important features to make the output more informative.

We next elaborate the core of transformer LLM, i.e., self-attention mechanism, via the illustrated Figures~\ref{fig:llminfer}(c) and (d).
In general, the self-attention mechanism involves two major steps: (i) measuring how relevant each of the previous context tokens is to the current token being processed; and (ii) combining the information from them into a single output vector.
A well-trained LLM has three projection matrices, i.e., a query projection matrix $\bm{W}_Q$, a key projection matrix $\bm{W}_K$ and a value projection matrix $\bm{W}_V$, which are used to calculate the attention.
In particular, the self-attention mechanism starts by multiplying the input matrix $\bm{X} \in \mathbb{R}^{n \times d}$, where $n$ is the number of input vectors and $d$ is the dimensionality of the embedding vector of each token, by the projection matrices to create three new matrices, i.e., query matrix $\bm{Q}$, key matrix $\bm{K}$ and value matrix $\bm{V}$, as shown in Figure~\ref{fig:llminfer}(c).
These three matrices are the information of the input tokens in three different spaces, which are used to calculate the attention.
In recent transformer LLMs (e.g., Llama 3.1), multi-query and multi-head self-attention mechanisms are employed to improve the scalability of larger models.
For simplicity, we utilize one self-attention head for illustration in Figure~\ref{fig:llminfer}(d) as every head of multi-head attention has a distinct version of matrices of queries, keys and values, see Figure~\ref{fig:llminfer}(c).
\begin{equation}\label{eqn:attention}
\centering
    z_{ij} = \frac{\bm{q}_{i} \cdot \bm{k}_j^{T}}{\sqrt{d}};~~~
    a_{ij} = \mathsf{softmax}(z_{ij});~~
    \bm{o}_{i} = \sum_{s = 1}^{i} a_{is} \cdot \bm{v}_s
\end{equation}
As shown in Figure~\ref{fig:llminfer}(d),  to generate the $i+1$-th token $t_{i+1}$, the self-attention mechanism in each head computes the \deng{inner product} between the query vector $\bm{q}_{i} \in \mathbb{R}^{1 \times d}$ and the key vector of the past tokens $\bm{k}_j$ where $j \in [1, i]$.
The computed product is scaled by $\sqrt{d}$ and normalized via a $\mathsf{Softmax}$ function to derive the attention score $a_{ij}$.
These attention scores multiply with the value vectors $\bm{v}_s$ in value matrix $\bm{V}$ to compute the output $\bm{o}_{i}$, 
see Equation~(\ref{eqn:attention}).

\stitle{LLM Inference Phases} 
In LLM services, the LLM inference procedure of a prompt can be decomposed into two phases: prefill phase and decode phase.
Specifically, in the prefill phase the LLM processes all the input tokens in user prompts and generates the first output token.
The service level objective (SLO) of the prefill phase in LLM service is its duration, i.e., \textit{Time-To-First-Token (TTFT)}.
In the decode phase the LLM sequentially generates the answers.
This phase completes when an end-of-sequence token \texttt{<eot>} is generated or when the context reaches a specified maximum length.
The SLO in the decode phase is \textit{Time-Per-Output-Token (TPOT)}.

\stitle{KV Cache}
Recall that the last generated token is appended to the previous context and then input into the LLM for the next token generation.
In particular, the new context does another forward pass of the model.
Obviously, the performance of the decode phase can be significantly improved by caching the key and value matrices of the previous context
(see KV cache in Figure~\ref{fig:llminfer}(d)) as they do not need to be recomputed.
The KV cache is one of the core components in the self-attention mechanism which is widely used in recent LLMs and offers significant speedup of the decode phase.

\stitle{Sparse Attention}
The attention calculation in Equation~(\ref{eqn:attention}) is the computationally expensive part of LLM inference.
To make matters worse, the key and value matrices consume large GPU memory space.
The sparse attention mechanism has been proposed to improve the efficiency of the attention calculation and reduce the GPU memory consumption during the LLM inference.
The intuition of sparse attention is that only a small proportion of tokens, not all tokens in previous context, dominates the generation quality/accuracy~\cite{RA}.
For example, only the key vectors and value vectors in KV cache with red rectangles in Figure~\ref{fig:llminfer}(d) are critical vectors for the high-quality token generation.
The computation cost of the attention calculation is significantly reduced as the sparse attention only calculates a fixed size of keys (resp. values) in key matrix $\bm{K}$ (resp. value matrix $\bm{V}$),  instead of all keys and values in both key and value matrices, see Equation~(\ref{eqn:attention}).
Key vectors with high inner product scores relative to the query vector are considered important, as they have high attention scores, significantly contributing to the final output.

\section{Motivation of AlayaDB}\label{sec:motivation}
The context length of an LLM request becomes very large with the rapid development of LLM applications.
For example, users may ask LLM questions about long documents, including understanding academic papers~\cite{explainpaper}, 
getting legal assistance from law documents~\cite{lawbench,llm4lawyers}, or analyzing financial documents~\cite{financial}.
Chat applications~\cite{chatgpt,kimi,deepseek} utilize the long chat log to produce better responses for users.
The AI programming assistants leverage all code in the project to accurately generate code or identify bugs/errors~\cite{copilot,cursor}.

The long-context LLM inference is extremely expensive as its self-attention mechanism incurs high memory consumption and numerous computation operations.
In particular,  it requires $O(n)$ memory to store the KV cache, where $n$ is the length of the long context.
The compute complexity for the prefill phase is $O(n^2)$ due to the self-attention computation in Equation~(\ref{eqn:attention}) that applies to each input token. Thus, the TTFT of prefill phase is several minutes to tens of minutes when the context length is quite large.
For the decode phase, it needs $O(n)$ for each token generation.
In practice, it takes about 141.38 GB memory and 6 minutes to answer a question about the book “\emph{Database System Concepts, 7th Edition}”~\cite{database-book} (495.5K tokens), with a bfloat16 version \textsf{Llama-3-8B} model~\cite{llama-1048k} on 2 NVIDIA A800 GPUs (each has 80 GB memory).
To reduce the memory consumption and computation cost of long context LLM inference, several research studies have been proposed to reuse the KV cache of the long context and use them to serve different requests from the users.
For example, the users may ask various questions about the same book “\emph{Database System Concepts}”. Thus, the KV cache of this book can be reused to answer these different questions.
The reused KV cache reduces the latency of TTFT in prefill phase significantly, and it becomes a de facto standard in LLM inference systems.
However, the performance of long-context LLM inference still has a lot of room for improvement. 
We next analyze the existing LLM inference systems/techniques in four dimensions: (i) GPU memory consumption, (ii) inference latency, (iii) generation quality, and (iv) solution usability.

\begin{table}
\centering
\small
\caption{LLM inference solutions analysis}\label{tab:solutions}

\begin{tabular}{|c|c|c|c|c|}
\hline
\textbf{Existing}   & \textbf{GPU memory}  & \textbf{Inference}  & \textbf{Generation} & \textbf{Solution} \\
\textbf{solution}      & \textbf{consumption} & \textbf{latency} & \textbf{quality} & \textbf{usability}  \\ \hline
\ding{172}          & Large           & High               & Good     & Good                \\ \hline
\ding{173}          & Large           & Medium             & Good     & Medium              \\ \hline
\ding{174}          & Small           & ---             & Medium   & Bad                 \\ \hline
\textbf{AlayaDB}       & \textbf{Small}  & \textbf{Low}    & \textbf{Good}   & \textbf{Good}       \\ \hline
\end{tabular}
\end{table}

\begin{figure}
    \small
    \centering
    \includegraphics[width=0.95\columnwidth]{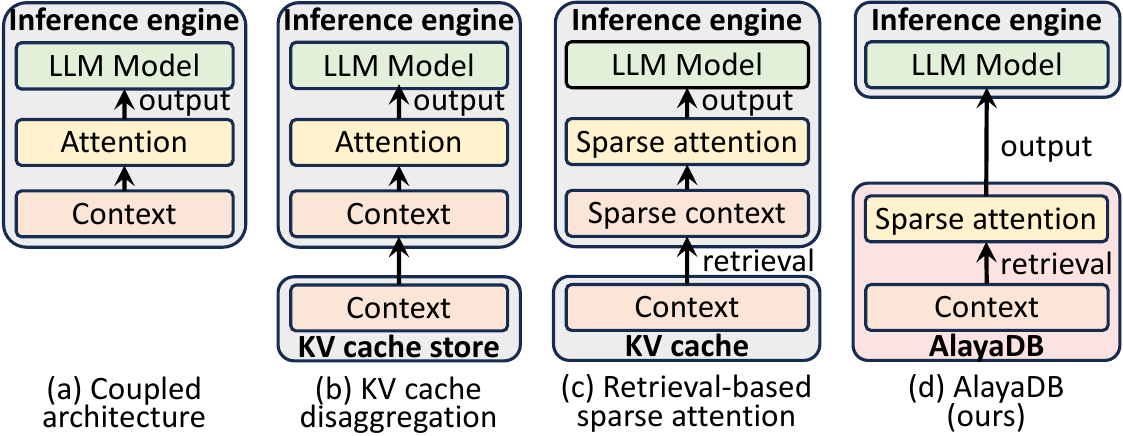}
    \trim
    \caption{Summary of LLM inference solutions}
    \label{fig:systems}
    \trim
\end{figure}

\subsection{Analysis of Existing Solutions}
In this section, we classify existing work into three categories: \ding{172} coupled architecture, \ding{173} KV cache disaggregation, and \ding{174} Retrieval-based sparse attention mechanism.
We introduce the core idea of each category and analyze the characteristics of them in detail.
Table~\ref{tab:solutions} summarizes the analyzed results of existing solutions.

\stitle{\ding{172} Coupled Architecture}
It is the widely-used LLM inference system architecture in industry, e.g., vLLM~\cite{vllm}, SGLang~\cite{sglang}, and transformers~\cite{transformers}.
The core idea of the coupled architecture is the LLM model computation and KV cache management are tightly coupled and it processes the user request in a holistic manner, as shown in Figure~\ref{fig:systems}(a).
It offers good usability with a simple user interface and high generation quality. 
However, it fails to handle long context. 
The major reasons are: (i) the large GPU memory consumption for KV cache 
and (ii) the high TTFT in prefill phase as it reuses the KV cache in a coarse manner, e.g., vLLM employs LRU policy to maintain the KV cache in limited GPU memory.

\stitle{\ding{173} KV Cache Disaggregation}
As depicted in Figure~\ref{fig:systems}(b),  several systems decouple the KV cache into a separate storage service and manage it in a stateful way.
For example, LMCache~\cite{lmcache} and Mooncake~\cite{mooncake} store the KV cache of a long context in external cheap storage  (e.g., CPU memory, disk or remote memory) after its prefill phase such that the KV cache can be reused by everyone in the future as it only needs to be loaded into the inference engine.
The inference latency of the KV cache disaggregation solutions is slightly lower than the coupled architecture as it reduces the TTFT of prefill phase by reusing KV cache better.
The generation quality of it is the same as the coupled solution as both employ full attention mechanism.
However, it is not easy to use, as it involves a lot of intrusive modifications (i.e., lots of engineering work) to the inference engine.
Moreover,  the KV cache disaggregation still consumes a large GPU memory during the decoding stage.

\stitle{\ding{174} Retrieval-based Sparse Attention}
Recently, InfLLM~\cite{infllm} and RetrievalAttention~\cite{RA} use the sparse attention mechanism to alleviate the high GPU memory consumption of these systems in both \ding{172} and \ding{173}.
In particular, they only retrieve a small subset of keys and values from offloaded KV cache for attention computation,
see Figure~\ref{fig:systems}(c).
Although these retrieval-based solutions can significantly reduce GPU memory consumption, almost all (if not all) of them are not easy to use as (i) the retrieval algorithm is hard-coded in the underlying specific LLM model and cannot be directly used on other LLM models and (ii) they lack the ability to manage and reuse the long contexts among different requests and inference engines.
Moreover, they trade off between memory consumption/inference latency and generation quality.
In particular, the generation quality of these methods is determined by the retrieved critical keys and values.
However, it is challenging to retrieve all the critical keys and values efficiently.
Existing work assumes that the number of critical vectors is fixed (i.e., $k$) and then retrieves top-$k$ critical keys and values from the offloaded KV cache.
This static method cannot achieve the good generation quality of \ding{172} and \ding{173}, as we will elaborate in Section~\ref{sec:query}.
Regarding inference latency, the retrieval-based sparse attention methods introduce extra overhead to identify the critical key and value vectors.
However, they gains benefits during the attention computation as only the selected critical keys and values will be used.
According to our internal experimental evaluation, there is no clear winner between the extra overhead and the reduced attention computation.
Thus, we use `---' in the inference latency column of \ding{174}, see Table~\ref{tab:solutions}.

\trim
\subsection{Design Goals of AlayaDB}
Motivated by the above limitations, we propose a novel architecture for efficient and effective long-context LLM inference by decoupling both the KV cache and sparse attention computation from the LLM inference engine.
In particular, we architect a vector database AlayaDB to manage the offloaded KV cache and compute the sparse attention for LLM inference, as illustrated in Figure~\ref{fig:systems}(d).
The design goals of AlayaDB are as follows.

\stitle{G1: Ease-of-use}
The first design goal of AlayaDB for long-context LLM inference engine is ease-of-use.
Thus, the user interface abstraction of AlayaDB should be simple and compatible with LLM inference engines.
With these abstractions, the LLM developers could use AlayaDB easily for efficient and effective inference in their LLM applications,  e.g., analogue to how web developers use traditional database systems in their web applications.

\stitle{G2: High Quality}
The second design goal of AlayaDB for long-context LLM inference engines is to provide high generation quality.
As mentioned above, the generation quality is determined by the quality of retrieved critical keys and values.
Thus, AlayaDB should offer the capability to identify the critical tokens.

\stitle{G3: Good Efficiency}
The third design goal of AlayaDB for long-context LLM inference engine is good efficiency.
Specifically, AlayaDB should achieve higher generation quality and lower memory consumption as much as possible for user-specified SLOs.

We are aware that many vector database systems/techniques have been proposed~\cite{milvus, manu, analyDBv, singlestore, alloydbai, diskann,  gaips, qsrp, faiss, pinecone, weaviate}, both in academia and industry.
However, to the best of our knowledge, none of them are natively designed to support efficient and effective long-context LLM inference.
In subsequent, we introduce the architecture and key components of AlayaDB.
As the last row in Table~\ref{tab:solutions} shows, AlayaDB incurs small memory consumption, low inference latency, and high generation quality simultaneously.

\section{System Overview of AlayaDB}\label{sec:alayadb}
Figure~\ref{fig:architecture} depicts the overview of AlayaDB we built at \textsf{AlayaDB.AI}.
It consists of three components: (i) user interface, (ii) query processing engine, and (iii) vector storage engine.
We briefly introduce each component in AlayaDB to elaborate on the designs for the aforementioned three design goals.

\begin{figure}
    \centering
    \small
    \includegraphics[width=0.95\columnwidth]{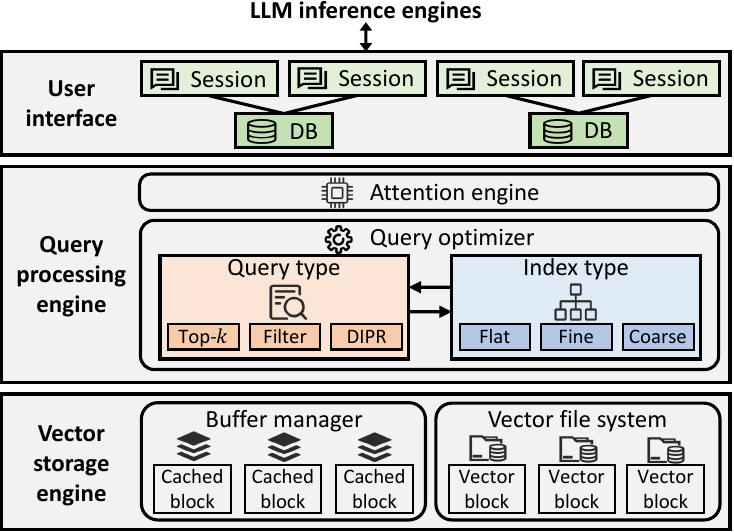}
    \caption{System overview of AlayaDB}
    \label{fig:architecture}
    \trim \trim \trim
\end{figure}

\stitle{User Interface}
The top layer of AlayaDB is the user interface component. 
It abstracts the complex attention computation and KV cache management to offer easy-to-use APIs.
Thus, LLM developers can simply leverage efficient and effective long-context LLM inferences by invoking the abstracted APIs in AlayaDB. 
This is similar to how web developers can build various applications without worrying about the underlying database management system.
Specifically, we use two widely-used concepts \verb|DB| and \verb|Session| in database community to abstract the \texttt{context} and \texttt{request} in the LLM inference procedure.
We will introduce the details in Section~\ref{sec:interface}.

\stitle{Query Processing Engine}
The middle layer of AlayaDB is the query processing engine, which is essential to achieve high quality and good efficiency goals.
It consists of a native attention engine and a query optimizer.
The native attention engine is designed for efficient sparse attention computation in Equation~(\ref{eqn:attention}).
The query optimizer is devised to identify the optimal query processing plan, which efficiently computes the critical tokens.
Unlike traditional database query optimizers, the query optimizer in AlayaDB has two major modules: (i) query type module, and (ii) index type module.
The query type module includes a set of predefined queries (e.g., top-$k$) that are used to retrieve the critical tokens from the KV cache.
The index type module has a set of indices that can be used to accelerate the predefined queries.
It is worth pointing out that both query type and index type in the query optimizer are extensible in AlayaDB.
The details of this engine are presented in Section~\ref{sec:query}.

\stitle{Vector Storage Engine}
To further improve the efficiency (both memory consumption and inference latency), we equipped AlayaDB with the vector storage engine in the bottom layer.
It includes a buffer manager and a novel vector file system.
A novel vector data layout scheme is designed in the vector file system, which could be used to improve the data access locality during query processing.
The buffer manager manages the buffered blocks of KV cache and supports high-performance keys and values retrieval.
We will show the optimizations of the vector storage engine in Section~\ref{sec:opt:storage}.

\section{User Interface}\label{sec:interface}

AlayaDB provides simple and flexible abstractions and easy-to-use APIs for users to \textsf{import context}, \textsf{reuse context} and \textsf{compute sparse attention result} for efficient and effective long-context LLM inference. 
Two core abstractions in AlayaDB are \verb|DB| and \verb|Session|. 
A \verb|DB| in AlayaDB manages all the contexts, including prompts, KV cache and vector indexes,
e.g., an analogue of \verb|DB| instance in traditional relational database systems, which include the schema, tables, and data tuples.
In a traditional database system, a database session is the connection established between an application server and a database server to enable communication and data retrieval.
Inspired by it, in AlayaDB, a \verb|Session| connects the contexts and the running inference requests from a user.
AlayaDB provides compatible APIs with HuggingFace transformers~\cite{transformers} and flash-attention~\cite{dao2022flashattention,dao2023flashattention2} library, which are the de facto standards of LLM inference and attention computation. 
The core APIs provided by AlayaDB are summarized in Table~\ref{tab:APIs}. We briefly introduce them as follows.

\begin{table}[]
\centering
\caption{AlayaDB APIs}\label{tab:APIs}
\begin{tabular}{|l|}
\hline
\textbf{DB abstraction and provided APIs}                              \\ \hline
\verb|DB.create_session(prompts) -> Session, prompts| \\ \hline
\verb|DB.import(prompts, kv_cache)| \\ \hline
\verb|DB.store(session)|\\ \hline
\textbf{Session abstraction and provided APIs}                         \\ \hline
\verb|Session.attention(q, layer) -> o|              \\ \hline
\verb|Session.update(q, k, v, layer) -> k, v|                \\ \hline
\end{tabular}
\end{table}

\squishlist

\item \verb|DB.create_session(prompts) | takes a list of prompts as input and returns a \verb|Session| object and the truncated prompts.
Given the input prompts, it reuses the longest common prefix with the stored contexts.
The reused context is in the \verb|Session| object.
The non-reused part of input prompts are the truncated prompts.
\item \verb|DB.import(prompts, kv_cache)| imports a list of computed contexts to AlayaDB for further reuse. 
Thus, its inputs are the prompts and KV cache of these contexts.
\item \verb|DB.store(session)| persists all states in a \verb|session| into a reusable context in the database.
It takes the \verb|session| {with the corresponding prompts and KV cache as the input.}

\item \verb|Session.attention(q, layer) -> o| generates the attention results of one LLM model layer for the session.
It accepts the query vectors and the layer id as the input, and returns computed attention output. 
This API can be used to replace the flash-attention APIs.

\item \verb|Session.update(q, k, v, layer) -> k, v| updates a session with the new inputs or generated tokens for one model layer. 
This API is compatible with \verb|DynamicCache.update| in huggingface transformers.
It provides an option to return the full key and value cache for manual management.

\squishend

\begin{figure} 
\centering
   \small
   \begin{tabular}{c} 	
       \includegraphics[width=0.95\columnwidth]{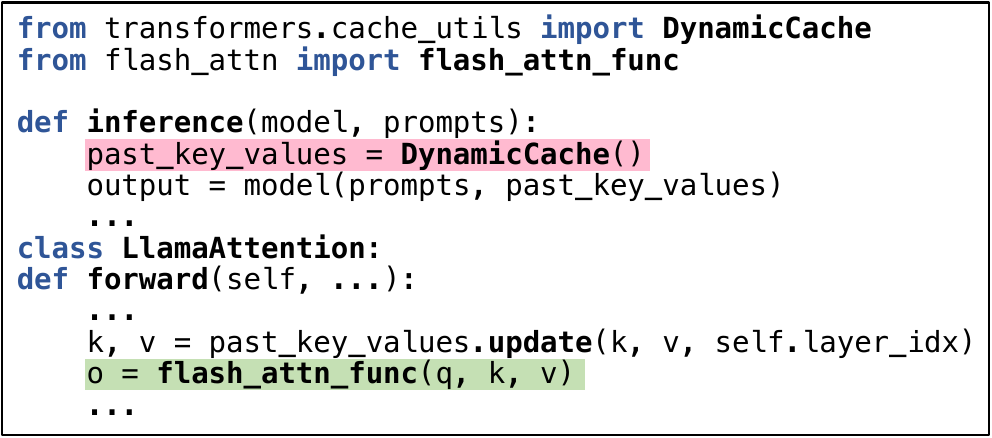}\\
       (a) Original code using flash-attention and transformers\\
       \includegraphics[width=0.95\columnwidth]{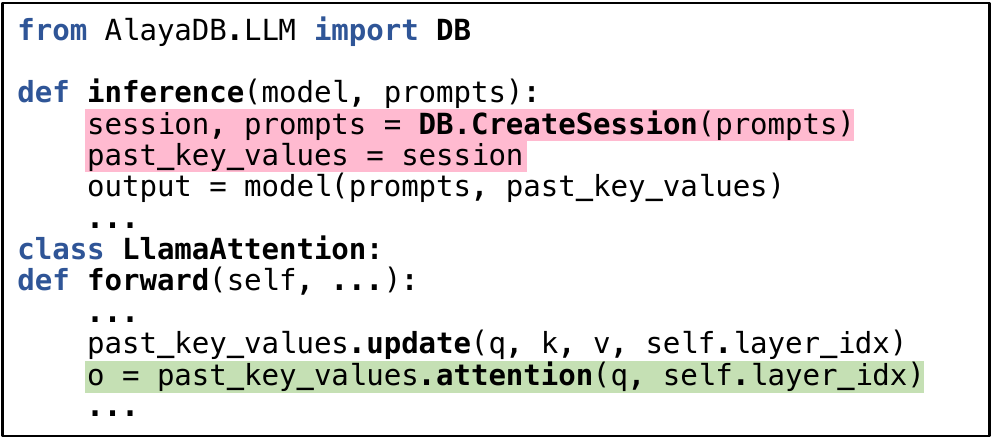}\\
       (b) Modified code using AlayaDB with transformers\\
   \end{tabular}
   \caption{Using AlayaDB APIs for LLM inference} \label{fig:apis}
   \trim \trim \trim
\end{figure}

\stitle{Example}
With the above APIs, it is easy for users to  \textsf{import context}, \textsf{reuse context} and \textsf{compute sparse attention score} for efficient and effective long-context LLM inferences upon various LLM models. 
Figure~\ref{fig:apis} shows an illustration example of AlayaDB with HuggingFace transformers, which only changes few lines of code.
In particular, Figure~\ref{fig:apis}(a) is the original code.
The \verb|inference| function is the common implementation of using an LLM model (offered by HuggingFace transformers).
For a model and a list of prompts, it creates a new \verb|DynamicCache| to manage the KV cache as the \verb|past_key_values|.
The prompts and \verb|past_key_value| are inputs of LLM model to generate the next tokens.
\verb|LlamaAttention.forward| is the implementation of an attention layer in HuggingFace transformers.
It first updates the \verb|past_key_value| with the newly generated key and value matrix, which is now a \verb|DynamicCache| with the full KV cache.
Then, it invokes \verb|flash_attn_func| attention operator on the newly generated query matrix and the full KV cache.
Figure~\ref{fig:apis}(b) shows how to use AlayaDB APIs for the above LLM inference procedure.
From the application side, users can enjoy the ability to manage and reuse the contexts in AlayaDB by simply replacing \verb|DynamicCache| with \verb|Session|, as the pink-colored lines show.
Specifically, it calls \verb|DB.CreateSession| to initialize a session and the truncated prompts for the input prompts.
The non-reusable parts (truncated prompts) are input to the LLM model together with the Session for further generation.
To further leverage the native attention computation from AlayaDB, users only need to modify the \verb|LlamaAttention.forward| to replace the flash-attention with the \verb|Session.attention|, see the last highlighted line in Figure~\ref{fig:apis}(b).

\section{Query Processing in AlayaDB}\label{sec:query}
In this section, we introduce the query processing procedure in AlayaDB.
In particular, we propose a novel query type in Section~\ref{sec:query:dipr}, which captures the dynamic nature of sparse attention in LLM inference.
In Section~\ref{sec:query:optimizer}, we introduce the query optimizer of AlayaDB.

\subsection{Dynamic Inner Product Range Query (DIPR)}\label{sec:query:dipr}

\begin{figure}
    \centering
    \small
    \includegraphics[width=0.95\linewidth]{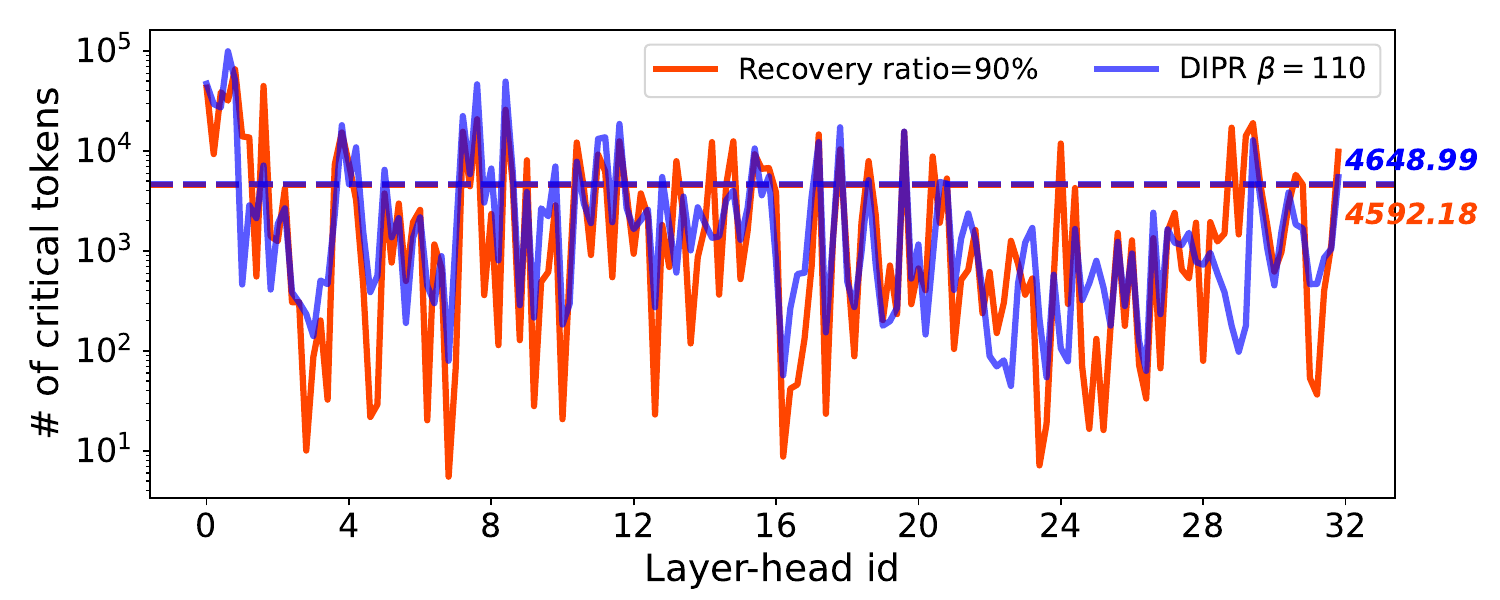}
    \trim
    \caption{The number of selected tokens in different heads}
    \label{fig:dynamic-heads}
    \trim
\end{figure}

\begin{table}
\small
\centering
\caption{The number $k$ of required tokens in different tasks}

\label{tab:tasks}
    \begin{tabular}{||c|c|c||c|c|c||}      \hline
      Task & k & proportion & Task & k & proportion\\ \hline \hline
        \textsf{Qasper} & 350 & 9.67\% & \textsf{LCC} & 65 & 5.26\% \\  \hline
        \textsf{Passage R.} & 250 & 2.69\% & \textsf{HotpotQA} & 200 & 2.19\% \\ \hline
        \textsf{QMSum} & 150 & 1.41\% & \textsf{TriviaQA} & 20 & 0.24\% \\       \hline
    \end{tabular}
\end{table}

\subsubsection{Limitations of Top-$k$ Query}\label{sec:query:dipr:moti}
In sparse attention, a subset of critical key and value vectors are retrieved to approximately compute the attention output.
Thus, the \textsf{effectiveness} of the computed attention output is determined by the number of retrieved critical tokens.
The \textsf{efficiency} of LLM inference also depends on the cost to retrieve these critical tokens.
Almost all (if not all) existing work~\cite{RA,infllm,adakv,quest,pyramidkv,duoattention,h2o} utilize top-$k$ query for critical token retrieval. 
The value of $k$ is a pre-defined hyper-parameter, and it is applied to all attention heads among all layers.
It means the top-$k$ query assumes the number of critical tokens is the same among \textit{different tasks} and \textit{different heads}.
However, this assumption is probably not true in various LLM applications.
We summarize two crucial observations as follows, which contradict this assumption.
These observations are summarized from the user experiences of our product's industry customers,
and we reproduce them in widely-used LLM benchmarks to follow the DeWitt Clause.

\stitle{Observation I: the number of critical tokens significantly varies in different \textit{heads}}
The transformer-based LLM model includes multiple layers and every layer has multiple heads.
We conduct an experiment with \textsf{Llama-3-8B-Instruct-262k} model~\cite{llama-262k} on the KV retrieval dataset in \textsf{$\infty$-Bench}~\cite{inf-bench} to investigate how many critical tokens are needed to result in a good approximation of the full attention scores.
We measure the accuracy of this approximation with the recovery ratio~\cite{RA}, which represents the proportion of the total attention scores accounted for by the attention scores of the selected critical tokens.
The red curve in Figure~\ref{fig:dynamic-heads} shows the number of tokens needed to achieve a recovery ratio of 90\% for each head (randomly sampled five heads per layer), which significantly varies among different heads. 
For example, it needs on average 42,979.85 tokens in layer 0 head 5, which is much larger than the 53.36 tokens in layer 31 head 5.

\stitle{Observation II: different \textit{tasks} require different number of critical tokens}
We conducted experiments on various tasks in \textsf{LongBench}~\cite{longbench} to explore the number of critical tokens for LLM inference in different tasks.
These tasks cover key long-text applications including \textsf{single-doc QA (Qasper)}, \textsf{synthetic tasks (Passage Retrieval)}, \textsf{multi-doc QA (HotpotQA)}, \textsf{summarization (QMSum)}, \textsf{code completion (LCC)}, and \textsf{fewshot learning (TriviaQA).}
Table~\ref{tab:tasks} lists the number of tokens $k$ (resp. its proportion to the context length) required for the top-$k$ query based sparse attention to achieve the same accuracy as full attention in these tasks. 
Obviously, the number of critical tokens varies widely across tasks, ranging from 20 (0.24\%) to 350 (9.67\%).
For the simple tasks (e.g., \textsf{TriviaQA}), they only need a few tokens as the answer can be obtained from a short paragraph of context.
In contrast, complex tasks (e.g., \textsf{Qasper}) require a large amount of tokens to understand the whole context then return correct answers.

\stitle{\underline{Take-away message}} The nature of sparse attention is to use a dynamic set of critical tokens to generate high-quality responses (w.r.t. full attention) in different tasks and different heads of the transformer-based LLM models.
The traditional top-$k$ query fails to capture the dynamic nature of sparse attention as it uses a fixed and static $k$,
which always results in either low generation quality (i.e., retrieving too few critical tokens) or high computation cost (i.e., retrieving too many critical tokens).

\subsubsection{From Attention to DIPR}\label{sec:query:dipr:deduce}
To overcome the limitation of the traditional top-$k$ query with static and fixed $k$ for different heads and tasks,
we propose Dynamic Inner-Product Range query (DIPR) to capture the dynamic nature of sparse attention.
In particular, DIPR adaptively determines the number of \textit{critical tokens} in different tasks and heads.
We first formally define the critical token \deng{considered by DIPR} in Definition~\ref{def:critical_token}.

\begin{definition}[Critical token]
\label{def:critical_token}
Give the definition of attention score in Equation~(\ref{eqn:attention}), considering all key vectors in the key matrix $\bm{K}= [\bm{k}_1, \cdots, \bm{k}_n ]$, the key $\bm{k}_j$ is a critical token for query vector $\bm{q}_i$ if and only if 
$a_{ij} \geq \alpha \times \max_{s \in [1,n]}(a_{is}),$
where $\alpha$ is a proportion threshold and ranges in $[0,1]$.
\end{definition}

The intuition of DIPR query is finding all tokens which are larger than a given proportion $\alpha$ of the token with maximum inner product as all these tokens are critical.
We next transform the critical token in Definition~\ref{def:critical_token} to an inner product-based version in Definition~\ref{def:ip_critical_token}.
Theorem~\ref{theorem:dipr} guarantees the correctness of the definition transformation.
Interestingly, Definition~\ref{def:ip_critical_token} means the DIPR query explicitly considers the attention computation in Equation~(\ref{eqn:attention}).

\begin{definition}[Inner Product-based Critical token]
\label{def:ip_critical_token}
Considering all key vectors in the key matrix $\bm{K}= [\bm{k}_1, \cdots, \bm{k}_n ]$, the key $\bm{k}_j$ is a critical token for query vector $\bm{q}_i$ if and only if 
$\bm{q}_i \cdot \bm{k}_j^T \geq \max_{s \in [1,n]}(\bm{q}_i\cdot\bm{k}^T_s) - \beta, \text{where  }  \beta = - \sqrt{d} \times ln(\alpha).$
\end{definition}

\begin{MyTheo} \label{theorem:dipr}
The critical token in Definition~\ref{def:critical_token} is equivalent to the inner product-based critical token in Definition~\ref{def:ip_critical_token}.
\end{MyTheo}

\begin{proof}
\small
\begin{align*}
& \quad a_{ij} \geq \alpha \times \max_{s \in [1,n]}(a_{is}) 
\Leftrightarrow \quad \frac{\exp(z_{ij})}{\sum_{t=1}^n\exp(z_{it})} \geq \alpha \times \max_{s \in [1,n]} \left(\frac{\exp(z_{is})}{\sum_{t=1}^n\exp(z_{it})} \right) \\
& \Leftrightarrow  \quad \exp(z_{ij}) \geq \alpha \times \max_{s \in [1,n]}(\exp(z_{is}))  
\Leftrightarrow \quad z_{ij} \geq \ln(\alpha) + \max_{s \in [1,n]}(z_{is})  \\
& \Leftrightarrow  \quad \bm{q}_i\cdot\bm{k}^T_j \geq \sqrt{d} \times \ln(\alpha) + \max_{s \in [1,n]}(\bm{q}_i\cdot\bm{k}^T_s) 
\end{align*}
The proof completes by setting $\beta$ as $- \sqrt{d} \times ln(\alpha)$.
\end{proof}

Last, we formally define the novel DIPR query in Definition~\ref{def:dipr}.

\begin{definition}[Dynamic Inner-Product Range Query, DIPR($\bm{q},\beta$)]
\label{def:dipr}
Given a key matrix $\bm{K}= [\bm{k}_1, \cdots, \bm{k}_n ]$, a query vector $\bm{q}_i$ and a parameter $\beta \geq 0$, 
the DIPR query returns a subset $\bm{cK}$ of $\bm{K}$, which includes all inner product-based critical tokens.
\end{definition}

The advantages of our novel DIPR query are three-fold:
(i) For a given $\beta$, different numbers of critical tokens will be retrieved by different tasks and heads in DIPR query.
Thus, it explicitly considers the dynamic nature of sparse attention;
(ii) the input parameter $\beta$ of DIPR query directly considers the critical tokens by the attention score of every key,
however, the top-$k$ query utilizes the absolute rank of every key's attention score;
and
(iii) the core computation of DIPR is the inner product $\bm{q}_i\cdot\bm{k}^T_j$, which does not introduce extra overhead and the optimizations for inner product-based top-$k$ query can be directly adopted.
We demonstrate the effectiveness of our novel DIRP query by the experiments in Figures~\ref{fig:dynamic-heads} and~\ref{fig:dynamic-workload}.
In particular, the blue curve in Figure~\ref{fig:dynamic-heads} shows the number of retrieved critical tokens of DIPR query by setting $\beta$ to $110$, which is very close to the number of tokens required to achieve 90\% recovery ratio.
In Figure~\ref{fig:dynamic-workload}, we present the obtained results by varying $\beta$ and $k$ in DIPR query and top-$k$ query for \textsf{Passage R.} and \textsf{LCC} tasks, respectively.
It confirms the DIPR query achieves higher accuracy with fewer retrieved tokens when compared with top-$k$ query.

\begin{figure}
    \small
    \centering
    \subcaptionbox{Passage R.}{\trim\includegraphics[width=0.5\linewidth]{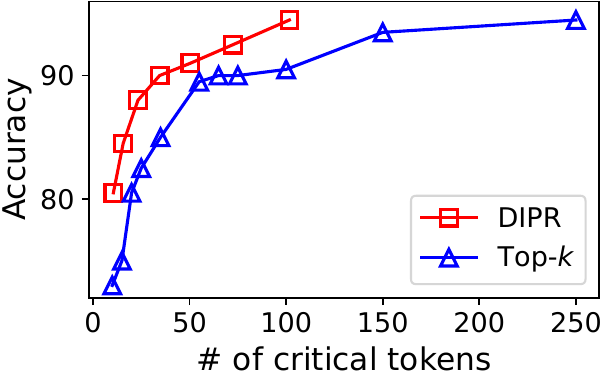}}
    \subcaptionbox{LCC}{\trim\includegraphics[width=0.47\linewidth]{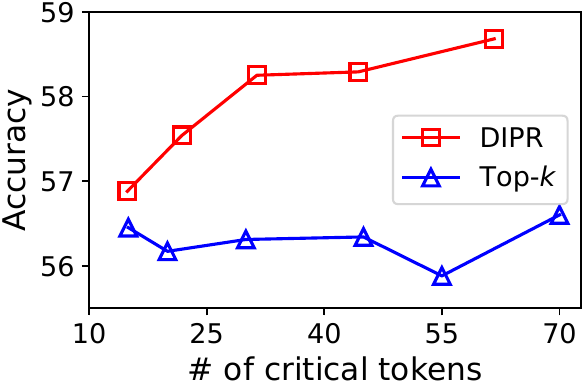}}
    \caption{The number of critical tokens in different tasks}    \label{fig:dynamic-workload}
\end{figure}

\subsubsection{DIPR Query Processing} \label{sec:query:dipr:process}
The top-$k$ query processing algorithms efficiently return a sized-$k$ set of critical tokens for every query vector by exploiting widely used graph indices on key vectors, e.g., HNSW~\cite{hnsw}, NSG~\cite{nsg} and RoarGraph~\cite{roargraph}.
However, they cannot be directly used to process DIPR query as DIPR query returns a variable length of critical tokens w.r.t. the maximum inner product value of the query $\bm{q}_i$ and key matrix $\bm{K}$ for different tasks and different heads.
In this section, we devise the first approximate DIPR query processing algorithm DIPRS.
There are two principles of DIPRS algorithm design: 
(i) it should explore more points to find the larger inner product value quickly; 
and (ii) it should reduce non-critical point explorations.

\autoref{algo:diprs} shows the pseudocode of the DIPRS algorithm, which follows the above two principles.
Specifically, it utilizes the widely used graph-based indices as the fundamental building block as they offer high recall and good efficiency for inner product-based vector similarity search.
Given an input $\beta$, the number of returned tokens in the critical token set $\bm{cK}$ is \textit{dynamic} and \textit{unknown} in advance, until the token with maximum inner product value is found.
The core ideas of DIPRS algorithm are (1) maintaining an unordered candidate list with variable capacity,
and (2) progressively reducing the search space with the best-so-far inner product value.
The subroutine \textsf{tryAppend} (Line 10) decides whether the given point $\bm{k}$ should be appended into candidate list or not.

We next briefly present how \autoref{algo:diprs} achieves both above intuitions with the illustration example in \autoref{fig:diprs}.
To achieve (i), we set a capacity threshold $l_0$. 
When the list capacity is lower than $l_0$, it explores all points without pruning (see Line 13).
As shown in \autoref{fig:diprs}(a), 2 is added to the list even though it is not critical.
For (ii), after reaching the capacity threshold, it does not append the non-critical points to the list to reduce the search space.
Figures~\ref{fig:diprs}(b) and (c) show that 3 is pruned and 7 is appended w.r.t. the current maximum inner product value, respectively.

\begin{algorithm}[t]
\small
\setstretch{0.85}
\caption{DIPRS($G$, $\bm{q}$, $\bm{k}_0$, $l_0$, $\beta$)}
\label{algo:diprs}
\SetAlgoNoEnd
\SetAlgoLined
\SetAlgoNlRelativeSize{0}
\DontPrintSemicolon
    \KwIn{Graph $G$, query $\bm{q}$, start key $\bm{k}_0$, capacity threshold $l_{0}$, and $\beta$}
    \KwOut{Critical token set $\bm{cK}$}
    \BlankLine
    Initialize a list $C$ with start key vector $\bm{k}_0$ \\
    $i \leftarrow 0$    \\
    \While{\upshape $i < C.$\textsf{capacity}$()$} {
        $\bm{c}_i \leftarrow$ the $(i+1)$-th key vector in $C$ \\
        $i \leftarrow i + 1$    \\
        \ForEach{\upshape unvisited neighbor $\bm{k}$ of $\bm{c}_i$ in $G$} {
            \textsf{tryAppend}$(\bm{q}, \bm{k}, \beta, C, l_0)$
        }
    }
    $\hat{\bm{c}} \leftarrow$ the closest point to $\bm{q}$ in $C$    \\
    \Return{$ \bm{cK} \leftarrow \{\bm{c} | \bm{c} \in C, \bm{q}\cdot\bm{c}^T \geq \bm{q}\cdot\hat{\bm{c}}^T - \beta\}$}

    \SetKwProg{Procedure}{Procedure}{:}{\KwRet}
    \Procedure{\upshape \textsf{tryAppend}$(\bm{q}, \bm{k}, \beta, C, l_0)$}{
        $\hat{\bm{c}} \leftarrow$ the closest point to $\bm{q}$ in $C$    \\
        Mark $\bm{k}$ as visited   \\
        \If{$C.$\upshape \textsf{capacity}$() \leq l_0$ or $\bm{q}\cdot\bm{k}^T \geq \bm{q}\cdot\hat{\bm{c}}^T - \beta$}{
            $C$.\textsf{append}($\bm{k}$) \\
        }
    }
\end{algorithm}

\subsection{Query Optimizer in AlayaDB} \label{sec:query:optimizer}
Except for the traditional top-$k$ query and our novel proposed DIPR query, we believe other auxiliary queries can be defined to achieve sparse attention, i.e., retrieving a subset of critical keys and values for high-quality generation.
However, the processing performance of these queries significantly varies among different hardware settings and workload characteristics.
Thus, it is crucial to provide a query optimizer in AlayaDB, which assists the LLM application developer in choosing the best query type with its underlying index structure.
In AlayaDB, we consider three query types (e.g., top-$k$, DIPR and filter query) and three index types (e.g., coarse-grained index, fine-grained index, and flat index). 
Interestingly, both query and index types can be extended in AlayaDB for efficient and effective sparse attention.
We next introduce the core idea of each index type and analyze their characteristics in Table~\ref{tab:index}.
\squishlist
\item \textbf{Coarse-grained index}.
It groups the adjacent tokens into blocks, where each block is represented by several vectors. 
It only computes the inner products between query and representative vectors during the retrieval 
and selects the critical blocks for attention computation.
This kind of algorithms includes InfLLM~\cite{infllm}, Quest~\cite{quest} and PQCache~\cite{pqcache}. 
These methods usually require a large GPU memory to cache the blocks and they can provide a very low latency for LLM inference.

\begin{figure} 
\centering
    \includegraphics[width=0.9 \columnwidth]{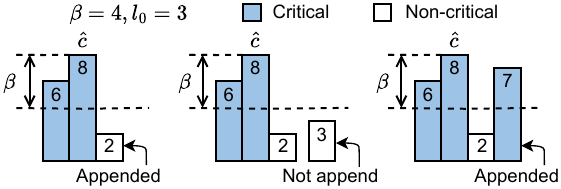}\\
   \begin{tabular}{ccc} 	
       \hspace{0.02\columnwidth} (a) $|C| \leq l_0$ & \hspace{0.1\columnwidth} (b) Point pruning & 
       \hspace{0.01\columnwidth} (c) $\bm{q}\cdot\bm{k}^T \geq \bm{q}\cdot\hat{\bm{c}}^T - \beta$ \\
   \end{tabular}
   \trim \trim 
   \caption{Three cases of \textsf{tryAppend} in DIPRS} \label{fig:diprs}
    \trim  \trim 
\end{figure}

\item \textbf{Fine-grained index}.
It builds the traditional vector search indexes on the key-level, e.g., indexing all key vectors by a graph (a.k.a., graph indices).
It quickly and accurately locates a small number of critical tokens in the index, which can be efficiently computed on CPU.
However, due to the expensive random memory access during index traversal, it can be slow when the number of used critical tokens is large, e.g., $k$ is large in top-$k$ queries.

\item \textbf{Flat index}. It scans all the keys to find the critical tokens on CPU. Compared to fine-grained indices, it is less efficient when the number of critical tokens are small due to redundant scans.
However, it can be more efficient when the number of critical tokens is large due to the sequential memory access.
\squishend

Inspired by the rule-based query optimizer in database systems, AlayaDB implements a unified and extensible optimizer to select an optimized query plan (including specified query type and index type) for attention computation. 
The workflow of the rule-based query optimizer in AlayaDB is shown in Figure~\ref{fig:optimizer}.
It identifies the context length at first.
Query to the short contexts will be processed directly with full attention. 
For the long contexts, if the context involves partial reuse, an attribute filtering predicate containing the length of the reused prefix is applied to the query, as we will introduce in Section~\ref{sec:opt:alg}.
Then the optimizer identifies GPU memory budget, which is set to the available GPU memory by default and can be manually set by users.
If the budget is enough, the query will be processed as top-$k$ queries with coarse-grained indices, i.e., InfLLM~\cite{infllm} in AlayaDB.
If the GPU memory budget is limited, the optimizer will choose DIPR query and select the index type based on the layer id.
From production environments of LLM inference and experimental evaluation benchmarks (see Figure~\ref{fig:dynamic-heads}), we observed that the first layer requires a large number of tokens to maintain the generation quality.
Thus, the optimizer of AlayaDB chooses flat indices for the first layer and uses graph-based DIPRS for the other layers.
It is still an open problem in optimizing the sparse attention with different query types and index types.
However, query optimization is widely studied in our database community, we hope the 
researchers in our community can solve it together. 

\begin{table}[]
\centering
\small
\caption{Characteristics of index types}\label{tab:index}
\begin{tabular}{|c|c|c|c|c|}
\hline
\textbf{\begin{tabular}[c]{@{}c@{}}Index\\ type\end{tabular}} & \textbf{\begin{tabular}[c]{@{}c@{}}Supported\\ query type\end{tabular}} & \textbf{\begin{tabular}[c]{@{}c@{}}GPU memory\\ comsumption\end{tabular}} & \textbf{\begin{tabular}[c]{@{}c@{}}Latency\\ small $k$\end{tabular}} & \textbf{\begin{tabular}[c]{@{}c@{}}Latency\\ large $k$\end{tabular}} \\ \hline
Coarse        & Top-$k$, Filter                                                      & Large                                                                     & Low                                                                & Low                                                                \\ \hline
Fine          & Top-$k$, Filter, DIPR                                                      & Small                                                                     & Low                                                                & High                                                               \\ \hline
Flat          & Top-$k$, Filter, DIPR                                                & Small                                                                     & Medium                                                             & Medium                                                             \\ \hline
\end{tabular}
\end{table}

\begin{figure}
    \small
    \centering
    \includegraphics[width=0.9\linewidth]{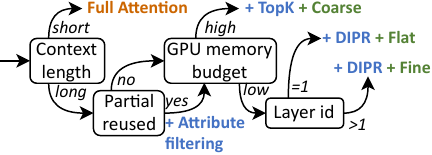}
    \ttrim
    \caption{Rule-based query optimizer in AlayaDB}
    \label{fig:optimizer}
    \trim \ttrim
\end{figure}

\section{Performance Optimization in AlayaDB}\label{sec:opt}

\subsection{Query Processing Optimization}\label{sec:opt:alg}
\stitle{Window Caching Enhanced DIPR}
Window caching retains a window of initial and last tokens during LLM inference, which is a standard technique in existing sparse attention algorithms~\cite{RA,streamingllm,infllm,snapkv,magicpig}.
The intuition of window cache is that those tokens usually contribute large attention weights.
AlayaDB adopts the window cache mechanism and caches the window in GPU memory.
Interestingly, the cached window can be further utilized to further enhance the quality of DIPR query results.
Recall that the core challenge of DIPR queries is to correctly identify the key vector with the maximum inner product value. 
Interestingly, our engineers observed that the key vector with maximum inner product value has a large probability in the cached windows.
For example, for dataset \textsf{math\_find} on the model \textsf{Llama-3-8B-Instruct-262k}, a window of 32 (initial) + 32 (last) tokens can cover almost 98\% of the key vectors with the maximum inner product values.
Motivated by this observation, we enhance DIPRS by taking the maximum inner product values in both the candidate list and the cached window into consideration.
It improves the performance of DIPRS by reducing the number of unnecessary tokens explored.

\stitle{Flexible Context Reuse By Attribute Filtering}
\label{sec:query:filter}
When a new session containing a full context is stored in AlayaDB, its vector index can be reused for efficient generation via sparse attention.
However, when a new session contains only a partial prefix of a stored context, the index cannot be reused.
This is a common case in practice. 
For example, a stored context contains a book and user A's conversation, while the incoming session of user B contains the same book but with new questions.
The new session only reuses the book, which is a partial prefix of the stored context.
In these cases, the session has to either be processed with expensive full attention or wait until a new index is built on the partial prefix.
To address the limitations, AlayaDB supports \textit{flexible context reuse}, which enables reusing the index of a stored context for efficient LLM inference when only a \textit{prefix} of the stored context is reused.
The challenge is to retrieve critical tokens only among the subset of tokens that are reused during searching within the full index.
Interestingly, the problem can be transformed into a well-studied problem in the database community called \textit{attribute filtering query} by considering the token id as an attribute.
The naive approach of attribute filtering is pruning those nodes that do not satisfy the attribute predicate.
However, this approach severely disrupts the connectivity of the graph index structure and leads to a significant decline in accuracy.
We improve DIPRS with a similar idea to~\cite{acorn} to solve this problem.
During each node exploration, the algorithm traverses both its neighbors and its neighbors' neighbors (2-hop neighbors).
Subsequently, the candidates that do not meet the filtering predicate are excluded.
This strategy enables AlayaDB to achieve a broader search scope during the retrieval process, which enjoys high efficiency and good accuracy.

\subsection{Computation Optimization}\label{sec:opt:compute}

\stitle{Index Construction Acceleration}
In AlayaDB, the index of a long context is constructed when a context is imported by \verb|DB.store()| and \verb|DB.import()|.
Although this procedure is usually offline, e.g., the book is pre-loaded before the service is launched, the cost is still not negligible.
We first analyze the overhead of index construction and then show how to optimize it.
The fine-grained index used in AlayaDB is RoarGraph~\cite{roargraph}, a state-of-the-art index for sparse attention~\cite{RA} due to its ability to handle vector search on Out-Of-Distribution (OOD) data.
Following RetrievalAttention~\cite{RA}, a RoarGraph is constructed for each query head, and the procedure can be divided into two stages:
(i) $\bm{q}$ to $\bm{k}$ kNN construction, which constructs a graph that links each query vector to its exact nearest key vectors, and
(ii) connectivity enhancement, which links each vector to its approximate nearest vectors that are produced by an ANNS search on the graph.
We observe that the overhead comes from the large number of indices and the slow kNN construction. 
We devise the following optimizations to reduce the overhead.

\sstitle{\textit{\underline{GQA-based index sharing}}} GQA is commonly used in state-of-the-art LLMs~\cite{llama} to reduce KV cache size.
It splits the $h_q$ query heads into $h_{kv}$ groups, where $h_{kv}$ is the number of key-value heads and $h_{kv} \textless h_{q}$. 
Queries heads within the same group will query the same key head, making one copy of KV cache able to be shared among a query group.
In AlayaDB, we share a RoarGraph among a query group by sampling query vectors from each query head and merging them into one RoarGraph in the stage of kNN construction.
In this way, the graph can still capture the distribution of all query heads while enjoying a speedup of ${h_q}/{h_{kv}}$ times by the reduction in the number of indices.
Our experiments show that index sharing only results in $\leq 3\%$ loss in top-$k$ recall, and does not affect the generation quality of end-to-end LLM inference.

\sstitle{\textit{\underline{GPU-based kNN construction}}} The kNN construction ~\cite{nn-descent} can be highly parallel, making it suitable for GPU.
We directly use the NVIDIA cuVS library~\cite{cuvs} to accelerate its construction on GPU.
To reduce the overhead of KV cache transfer, we process one layer at a time, which is compute-bound, and overlap it with the asynchronous CPU-GPU transmission in a pipeline manner.

\stitle{Late Materialization for Index Updating}
For each session, there are new KV caches generated from user inputs and model outputs.
It raises a design choice about when to physically update them into the context.
A straightforward solution is inserting the new KV cache to the existing index immediately after a new token is generated or input by users.
However, it significantly (i) increases the TPOT by the blocked index updating, and (ii) occupies two memory copies by maintaining a physical index for every session.
To address this problem, AlayaDB adopts a late-materialization strategy for index updating that does not affect the SLO.
By default, the newly generated KV cache is appended to the local window for retrieval.
The session will only be materialized into a new physical index when the \verb|DB.store()| API is explicitly called.
This is based on two practical observations that the user prompt and LLM generation following the long context are
(i) often short and (ii) often not reused across sessions.
Therefore, there is no need to early materialize the newly generated KV cache to the physical index in most cases.

\stitle{Data-centric Attention Engine}
AlayaDB is integrated with a native attention engine, which is optimized with data-centric computation.
Instead of computing attention after gathering the retrieved vectors~\cite{infllm,pqcache},
AlayaDB directly applies attention to the vectors where they reside, and then aggregates the attention results.
This data-centric mechanism can reduce the overhead of moving the large KV cache across different computing devices.
For example, when most of the context is on CPU and a window is cached on the GPU, partial attention of the two parts is computed independently in parallel and aggregated into a final attention output.
We use the same algorithm as FlashAttention~\cite{dao2022flashattention} and RetrievalAttention~\cite{RA} to compute and aggregate the partial attention outputs.

\subsection{Storage Optimization}\label{sec:opt:storage}
During LLM inference, AlayaDB retrieves a specific portion of vector data from each attention head of different attention layers to generate the next token. 
However, storing all the data in the limited CPU is not practical due to the large KV cache size.
To efficiently manage and reuse these vector data, we devise a vector file system and a purpose-built buffer manager within AlayaDB. 

\stitle{Vector File Systems} 
The vector file system in AlayaDB is built upon SPDK (Storage Performance Development Kit) to manage multiple vector files on disk in user space.
Specifically, each vector file stores the vectors of an attention head in a specific layer.
These stored vectors are organized into blocks, where vector indices and vector data are stored separately in different types of blocks, and vector index blocks are linked together in a graph structure.
The benefits of our layout are two-fold: (i) the graph-based structure allows for quick traversal and access to related vectors, and (ii) the vector data can be inserted or deleted without the need for restructuring the entire file.
Furthermore, the system can bypass traditional kernel I/O paths by leveraging SPDK, which significantly reduces latency and improves throughput.

\stitle{Purpose-built Buffer Manager}
AlayaDB has a purpose-built buffer manager built upon the underlying vector file system, which is designed to efficiently process the frequently used data in memory.
It employs the eviction strategy based on the corresponding block types.
For example, blocks storing the vector indices for attention heads are more likely to be kept in memory, as these vectors are frequently accessed during inference.
In contrast, blocks storing the vector data are only fetched once to calculate the attention score for each token.
The specific designs of it minimize redundant I/O operations by avoiding the need to retrieve them from secondary storage repeatedly.
Additionally, the buffer manager supports parallel access, enabling efficient processing in a multi-threaded environment.
\trim

\section{Use Cases of AlayaDB}\label{sec:applications}
AlayaDB provides easy-to-use interfaces and good performance for long context management and inference.
In this section, we present two LLM applications to demonstrate the use cases of AlayaDB.

\stitle{Financial Document Analysis}
AlayaDB can be used by financial companies to assist in their financial document analysis.
These documents are long, including financial statements, audit reports, business plans, etc.
Data analysts in the financial company leverage domain-specific LLMs with AlayaDB to analyze a large number of financial documents and generate summarizations for their purposes, e.g., the top-$10$ news of Hong Kong stock market in 2024.
The cost and latency of the document analysis service are reduced.

\stitle{Legal Assistant for Question Answering}
Law firms can utilize AlayaDB to enhance their intelligent legal assistant service.
The major difference between the legal assistant and other LLM applications is that answers to users' questions must be precise and accurate, e.g., comply with the rules of the government.
The legal documents can be stored as context in AlayaDB. 
Their domain-specific LLM answers user questions by the stored context to achieve low costs while guaranteeing result accuracy.

\section{Empirical Evaluation}\label{sec:exp}

\begin{table*}
  \centering
  \begin{center}
  \caption{Generation quality of different sparse attention algorithms in $\infty$-Bench. Each method used the number of [\textit{\textsf{initial}}+\textit{\textsf{last}}]+\textit{\textsf{retrieved}} tokens for attention computation.} 
  \trim 
    \label{tab:accuracy}
    \begin{tabular}{c|c|c|c @{\hspace{1.5\tabcolsep}} c @{\hspace{1.5\tabcolsep}} c @{\hspace{1.5\tabcolsep}} c @{\hspace{1.5\tabcolsep}} c @{\hspace{1.5\tabcolsep}} c @{\hspace{1.5\tabcolsep}} c @{\hspace{1.5\tabcolsep}} c|c}
        \Xhline{1px}
        Methods & Setting & SLO & \textsf{Retr.KV} & \textsf{Retr.P} & \textsf{Retr.N} & \textsf{Code.D} & \textsf{En.MC} & \textsf{En.QA} & \textsf{En.Sum} & \textsf{Math.F} & Avg. \\
        \hline
        Full Attention & --- & \ding{55} & 15.8 &	100.0 &	100.0 &	27.4 &	55.9 &	31.0 &	15.1  &	19.1 & 45.6 \\

        InfLLM & [128+4K]+4K tokens & \ding{51} & \textbf{25.0} & 100.0 &	100.0 &	28.2 &	39.7 &	18.7 &	15.3  &	23.4 & 	43.8 \\

        StreamingLLM & [128]+8K tokens & \ding{51} & 3.8 &	8.5 &	8.5 &	27.7 &	41.5 &	14.5 &	14.3  &	16.3 & 16.9 \\

        Top100 & [128+512]+100 tokens & \ding{51} & 6.6 & 100.0 & 100.0 & 30.0 & 56.3 & 29.7 & 15.2  & 24.6 & 45.3 \\

        Top2000 & [128+512]+2K tokens & \ding{55} & 14.6 & 100.0 &100.0 &	29.7 &	58.1 &	31.2 &	16.0  &	24.3 &	46.7 \\

        DIPRS & [128+512] tokens, $\beta=50$ & \ding{51} &  14.0 &	\textbf{100.0} &	\textbf{100.0} &	\textbf{30.7} &	\textbf{58.1} &	\textbf{32.1} &	\textbf{16.4} &	\textbf{24.9} &	\textbf{47.0} \\
        \Xhline{1px}
    \end{tabular}
  \end{center}
\end{table*}

In this section, we conduct experiments to evaluate the end-to-end performance of AlayaDB in long-context LLM serving.
In particular, we aim to answer the following two questions:
\squishlist
\item \textbf{Q1: Can AlayaDB achieve low latency, high quality, and low resource consumption at the same time for long-context LLM serving?} (Section~\ref{sec:exp:e2e})

\item \textbf{Q2: How is the effectiveness of our proposed performance optimizations in AlayaDB?} (Section~\ref{sec:exp:opt})
\squishend

\stitle{Hardware Configuration}
We conduct our experiments on a server with one NVIDIA L20 GPU (48GB memory) and two Intel XEON GOLD 6542Y CPUs with 48 cores, 96 threads and 512 GB DRAM in total.
We use AlayaDB together with HuggingFace Transformers~\cite{transformers} to support LLM inference.
We use the bfloat16 version of Llama-3-8B-Instruct-262k~\cite{llama-262k}, the long context variant of a state-of-the-art LLM model Llama~\cite{llama} for inference.
The model has 32 layers. Each layer includes 32 query heads and 8 key value heads.
Its weights occupy 15.4 GB GPU memory during inference.

\subsection{End-to-end Performance Evaluation}\label{sec:exp:e2e}
\subsubsection{TPOT, Quality and GPU Memory Consumption}
We compare our proposed DIPR query (see Section~\ref{sec:query:dipr}) with existing sparse attention algorithms and full attention algorithm w.r.t. \textit{Time-Per-Output-Token} (a.k.a., TPOT, the inference latency per token generation), quality, and GPU memory consumption in various LLM inference workloads.

\stitle{Tested Workloads}
We adopt a widely-used long-context benchmark $\infty$-Bench~\cite{inf-bench} for overall performance evaluation.
Specifically, we use 8 tasks in $\infty$-Bench including \textsf{Retr.KV, Retr.P, Retr.N, Code.D, En.MC, En.QA, En.Sum, Math.F}.
The average input context length of different tasks ranges from 43.9K to 192.6K tokens.
In the experiments, the index of the input context is built in advance and we only measure the latency of each token generation (TPOT).
We set the SLO of TPOT $\leq 0.24$s, which is the reading speed of human~\cite{distserve}.

\stitle{Compared Methods} We compare the following methods:
\squishlist
\item \textsf{Full Attention}, which stores the KV cache of full context and computes the full attention on GPU. 
\item \textsf{InfLLM}~\cite{infllm}, it is a coarse-grained algorithm which selects critical tokens in blocks and computes their attention on GPU.
\item \textsf{StreamingLLM}~\cite{streamingllm}, it is an algorithm that keeps a window of tokens in GPU memory for attention computation and simply drops the other tokens.
\item \textsf{Top-$k$}, it is a fine-grained algorithm which processes the top-$k$ similarity search with graph-based index on CPU.
We follow the RetrievalAttention~\cite{RA} to use RoarGraph~\cite{roargraph} as the index and align the window size.
In particular, the parameter $k$ is set as 100 and 2000 to study the performance of difference retrieved critical tokens in our experiments.
\item \textsf{DIPRS}, our proposed DIPR query processing algorithm for sparse attention. It also uses RoarGraph as the index.
The window size of DIPRS is the same as that of the top-$k$ query.
\squishend

\stitle{Result Analysis}
Table~\ref{tab:accuracy} shows the generation quality of different methods in all 8 tasks of $\infty$-Bench.
The quality score is measured by $\infty$-Bench.
First of all, our proposed \textsf{DIPRS} not only guarantees the SLO, but also achieves the best average generation quality among all the compared methods, as the last column in Table~\ref{tab:accuracy} shows.
{Moreover, it is the overall winner in 7 tasks out of the 8 tested tasks.}
For full attention, the SLO of TPOT is violated due to the expensive O($n$) computation cost even with the KV cache in GPU memory.
Compared with full attention, DIPRS can achieve a near or even higher quality in all tasks.
The result also confirms the superiority of DIPRS against traditional top-$k$ query.
Top-$k$ requires retrieving 2000 tokens to achieve a similar quality to DIPRS, but fails to meet the SLO because of retrieving too many tokens.
The generation quality of top-$k=100$ is worse than DIPRS in 6 tasks.
In \textsf{Retr.P} and \textsf{Retr.N}, both DIPRS and top-$k=100$ have the same performance.

To answer Q1, we perform in-depth analysis on two tasks (i.e., \textsf{EN.MC} and \textsf{EN.QA}) w.r.t. the generation quality and GPU memory consumption with user specified SLO.
We vary the number of cached tokens for InfLLM and StreamingLLM to investigate the relationship of generation quality and GPU memory consumption.
For top-$k=100$ and DIPRS, we use the same settings in Table~\ref{tab:accuracy}.
As Figure~\ref{fig:slo} shows, compared to all the other methods, DIPRS achieves the best generation quality and lowest GPU memory consumption while guaranteeing the SLO of TPOT.
Regarding the coarse-grained methods InfLLM and StreamingLLM, a large GPU memory is required to achieve higher accuracy, which limits the throughput of online serving and makes them impractical to run on the consumer-grade GPU, e.g., NVIDIA GTX4090 (24GB memory).
Compared to top-$k$, the generation quality of DIPRS surpasses top-$k$ due to its ability to identify the dynamic number of critical tokens for efficient sparse attention.

\begin{figure}
\small
    \centering
       \begin{tabular}{cc} 	
        \hspace{-0.03\columnwidth}\includegraphics[width=0.93\columnwidth]{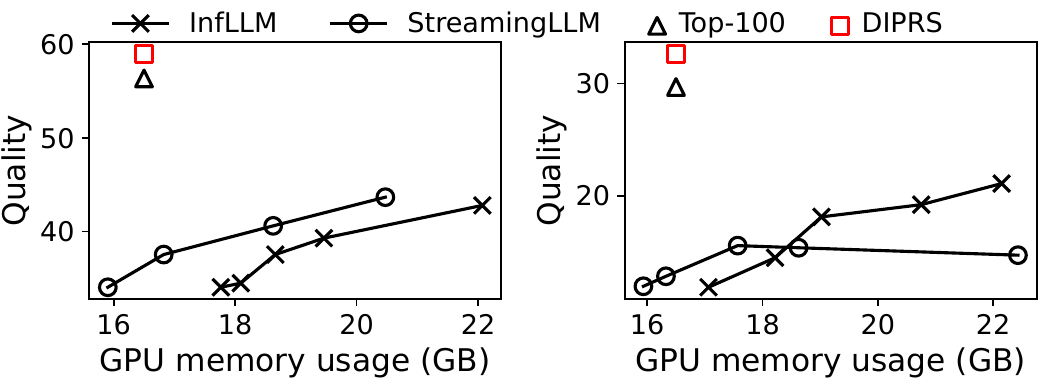}    \\
       \hspace{0.01\columnwidth} (1) EN.MC \hspace{0.32\columnwidth} (2) EN.QA\\
   \end{tabular}
   \trim
    \caption{Generation quality and GPU memory consumption with SLO guarantees}
    \label{fig:slo}
\end{figure}

\begin{figure}
\small
    \centering
       \begin{tabular}{cc} 	
       \includegraphics[width=0.50\columnwidth]{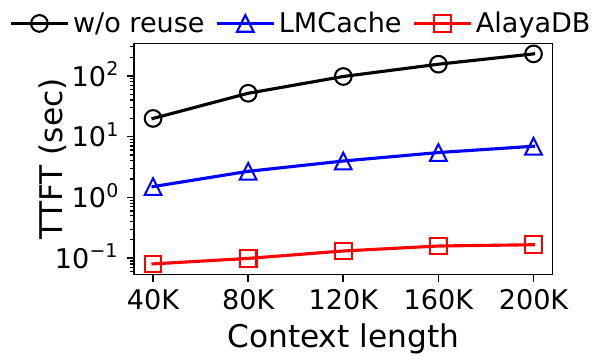} & \includegraphics[width=0.42\columnwidth]{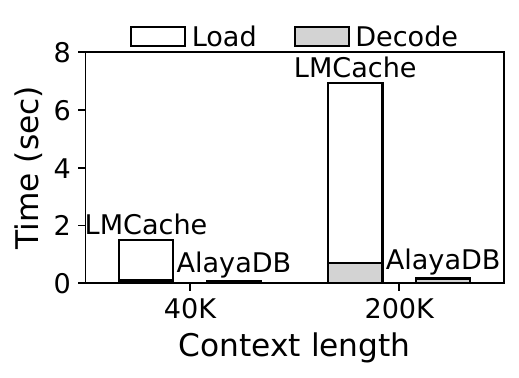} \\
       (a) TTFT & (b) Latency breakdown\\
   \end{tabular}
   \trim
    \caption{TTFT of long context reusing}
    \label{fig:ttft}
    \trim
\end{figure}

\subsubsection{\textit{Time-To-First-Token}: TTFT}
We compare AlayaDB with the state-of-the-art disaggregated KV cache service LMCache~\cite{lmcache,cachegen} to evaluate its ability to reduce the TTFT by efficiently reusing the stored long context in Figure~\ref{fig:ttft}.
LMCache stores the compressed KV cache of the full context, and supports context reusing by loading the KV cache into GPU.
In this experiment, we store the context in CPU memory in advance and measure the time of decoding the first token on this offloaded context as TTFT.
Figure~\ref{fig:ttft}(a) depicts the experimental results.
Firstly, reusing the KV cache is faster than recomputing the expensive prefill stage without reuse.
For example, our AlayaDB outperforms the w/o reuse by 2 to 3 orders of magnitude, see the red and black curves in Figure~\ref{fig:ttft}(a).
Secondly, the TTFT of AlayaDB is 19 to 42 times faster than LMCache, see the red and blue curves in Figure~\ref{fig:ttft}(a).
By analyzing the breakdown of latency of LMCache and AlayaDB in Figure~\ref{fig:ttft}(b), LMCache suffers from the slow KV cache loading, including decompressing and transferring from CPU to GPU.
The loading time increases linearly with the context length.
Instead of loading the KV cache, AlayaDB can directly decode on the offloaded KV cache with an extremely low latency, 
thus, resulting to a low TTFT for context reuse.
This experiment also confirms the analyzed limitation of the existing KV cache disaggregation architecture.
In other words, decoupling both attention computation and KV cache from the LLM inference engine as our proposal AlayaDB is a new opportunity for our community to develop fast and accurate LLM inference systems, which provides a huge optimization space.

\subsection{Effectiveness of Optimization Techniques}\label{sec:exp:opt}
\subsubsection{Index construction}
We conduct an ablation study of our proposed optimizations for the RoarGraph construction in Section~\ref{sec:opt:compute}.
In particular, we set the ratio of sampled queries for each index to 40\%, which means when building an index, the number of used query vectors is 40\% of the key vectors.
Figure~\ref{fig:build_index}(a) shows the index construction time under different context lengths.
The baseline method follows RetrievalAttention, which builds the index on CPU and builds one index for each query head, see the black curve.
Introducing GPU to build kNN and employing CPU-GPU pipeline can gain a speedup from 3$\times$ to 15$\times$, see the blue curve.
Then, by sharing the index in the same query group, index construction time can be further reduced from 12$\times$ to 62$\times$ compared to pure CPU baseline, as the red curve shows.
Moreover, index sharing also significantly reduces memory consumption by reducing the number of indexes.
As depicted in Figure~\ref{fig:build_index}(b), the index size can be 4$\times$ smaller than the GPU and CPU baseline without index sharing.

\begin{figure}
    \small
    \centering	
    \includegraphics[width=0.92\columnwidth]{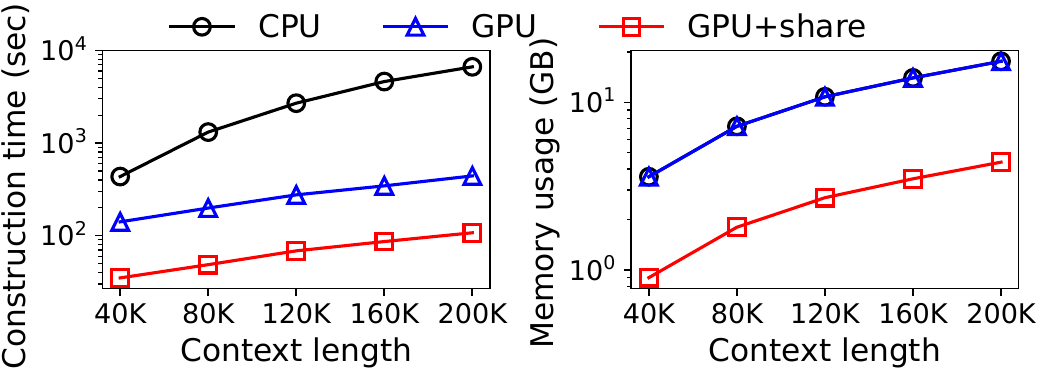} \\
    \hspace{0.1\columnwidth}  (a) Construction time  \hspace{0.15\columnwidth} (b) Memory consumption \\
    \trim
    \caption{Index construction optimization}
    \label{fig:build_index}
    \trim 
\end{figure}

\subsubsection{Filter-based DIPRS}
In Section~\ref{sec:opt:alg}, we introduce that AlayaDB leverages the attribute filtering with DIPRS algorithm to support partial context reuse.
In this experiment, we study the effect of this optimization to the generation quality and inference latency.
In particular, we conduct a micro-benchmark to evaluate the recall and latency of filter-based DIPRS search in the case of partial context reuse.
We fix the reused prefix length to 40K, and range the reuse ratio from 100\% to 20\% by varying the length of the stored context, i.e., the index size.
This micro-benchmark uses the KV cache generated by all heads in layer 1 during the \textsf{En.QA} task.
The 100\% reuse ratio means the stored context is fully reused, in other words, the filter-based DIPRS is the same as the original DIPRS without attribute filtering.
Figure~\ref{fig:filter} shows the measured recall and latency.
Firstly, the recall of filter-based DIPRS remains high with different reuse ratios, which guarantees the generation quality with partial context reuse in AlayaDB.
Secondly, when searching in a larger context with the same prefix length, the latency of filter-based DIPRS increases only slightly .
For example, the latency to search in 200K long context is only 1.13 ms higher than it is of 40K long context. 
Thus, AlayaDB guarantees the inference latency with good generation quality when partial context reuse is enabled.

\begin{figure}
    \centering	
       \includegraphics[width=0.85\columnwidth]{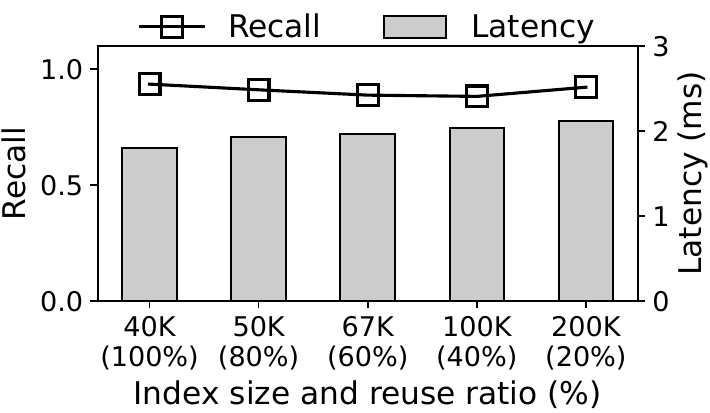}
       \trim
    \caption{Micro-benchmark of filter-based DIPRS}
    \label{fig:filter}
    \trim 
\end{figure}

\section{Conclusion} \label{sec:con}
At \textsf{AlayaDB.AI}, we built AlayaDB for efficient and effective long-context inference in LLM era.
From the architecture perspective, AlayaDB decouples the KV cache and attention computation from the LLM inference systems, and encapsulates them into a novel vector database system.
It optimizes the overall performance by co-optimizing attention computation and KV cache management in a monolithic manner. 
Collaborating with the inference engine, AlayaDB is able to guarantee the SLO while enjoying low resource consumption and high generation quality for long-context LLM inference.
The novel architecture poses new challenges and opportunities, including 
(i) implementing different parallelism strategies to enable distributed inference,
(ii) supporting more LLM inference engines like vLLM and SGLang, 
(iii) improving the query processing methods (or sparse attention algorithms) and query optimizer,
(iv) leveraging various storage tiers to store the KV cache of contexts, 
\deng{(v) utilizing heterogeneous hardware to accelerate the attention computation,
and (vi) designing attention-hybrid architecture for general-purpose vector databases.
}
We hope the researchers from different communities (e.g., database, machine learning, system) could tackle them together in the future.

\balance
\bibliographystyle{ACM-Reference-Format}
\bibliography{ref}

\end{document}